\documentclass[11pt]{article}
\usepackage{fullpage}
\usepackage{amsmath,amsthm,amsfonts,amssymb,mathtools}
\usepackage{hyperref}
\usepackage{cleveref}  
\usepackage[round]{natbib}

\usepackage{csquotes}
\MakeOuterQuote{"}

\newcommand\E{\mathbb{E}}
\newcommand\R{\mathbb{R}}
\newcommand\X{\mathcal{X}}
\newcommand\Y{\mathcal{Y}}
\newcommand\F{\mathcal{F}}

\newcommand\bW{{\mathbf{W}}}
\newcommand\bd{{\mathbf{d}}}

\newcommand\bu{{\mathbf{u}}}
\newcommand\bx{{\mathbf{x}}}

\mathtoolsset{centercolon}

\DeclarePairedDelimiter\ip{\langle}{\rangle}
\DeclarePairedDelimiter\norm{\|}{\|}
\DeclarePairedDelimiter\del{\lparen}{\rparen}
\DeclarePairedDelimiter\sbr{\lbrack}{\rbrack}
\DeclarePairedDelimiter\cbr{\{}{\}}

\DeclareMathOperator{\poly}{poly}

\newtheorem{theorem}{Theorem}
\newtheorem{fact}{Fact}

\title{\textbf{Dimension lower bounds for linear approaches to function approximation}}

\author{Daniel Hsu}
\date{}

\begin{document}

\maketitle

\begin{abstract}
  This short note presents a linear algebraic approach to proving dimension lower bounds for linear methods that solve $L^2$ function approximation problems.
  The basic argument has appeared in the literature before (e.g., Barron, 1993) for establishing lower bounds on Kolmogorov $n$-widths.
  The argument is applied to give sample size lower bounds for kernel methods.
\end{abstract}

\section{Introduction}

Function approximation is an important problem in many areas, and it is increasingly important for the methods of approximation to be computationally tractable when they are to be used in applications.
Linearity is a property that has both enabled the development of efficient algorithms for approximation, as well as the tractable mathematical analyses of such "linear methods" (defined below).
However, it has also been recognized that linear methods may be severely limited for solving certain approximation problems.
The purpose of this note is to demonstrate such limitations through a simple dimension argument.

We are primarily concerned with $L^2$ approximation of functions.
Let $\X$ be a domain (typically a subset of $\R^d$), $P$ be a probability distribution on $\X$, and $L^2(P)$ be the space of real-valued functions on $\X$ that are square-integrable with respect to $P$.
For any class of functions $\F \subseteq L^2(P)$, a \emph{linear method} for approximating functions from $\F$ is one that commits to choosing the approximation from a subspace $W \subseteq L^2(P)$ before getting any information about the target function from $\F$.
This is the setup behind the concept of Kolmogorov $n$-widths, and the argument given in this note is largely based on a lower bound by \citet[Lemma 6]{barron93universal} for the Kolmogorov $n$-width of a certain class of functions.
We do not know the lineage of this argument, but it has recurred in the literature several times in related contexts~\citep[e.g.,][]{blum1994weakly,kamath2020approximate,daniely2020learning,hsu2021approximation}.
We present a result from \citet{hsu2021approximation} in a slightly more general form to establish a lower bound on the dimension of the subspace used by any linear method that is able to achieve small approximation error with respect to $\F$.
The bound is given in terms of the number of near-orthogonal functions contained in $\F$.

We use the dimension lower bound to give a sample size lower bound for kernel methods~\citep{scholkopf2002learning}.
There are many lower bounds for kernel methods in the literature~\citep[e.g.,][]{ben2002limitations,warmuth2005leaving,khardon2005maximum,wei2019regularization,kamath2020approximate,allen2020backward}.
Our goal is to simply show how such a lower bound follows easily from the dimension lower bound, and also to point out an aspect of kernel methods as they relate to learning with non-adaptive membership queries.

\section{The dimension lower bound}

The following theorem is from \citet[Theorem 29]{hsu2021approximation} in a slightly more specialized form.
(Also see \citealp[Theorem 19]{kamath2020approximate} for a very similar theorem.)

\begin{theorem}
  \label{thm:dimension}
  Let $H$ denote a Hilbert space with inner product denoted by $\ip{\cdot,\cdot}_H$ and norm denoted by $\norm{\cdot}_H$.
  Fix any $\varphi_1,\dotsc,\varphi_N \in H$ with $\|\varphi_i\|_H^2 = 1$ for all $i=1,\dotsc,N$.
  Let $\bW$ be a finite-dimensional subspace of $H$ (and $\bW$ is allowed to be random) with $r := \E[ \dim(\bW) ] < +\infty$.
  Define
  \begin{align*}
    \epsilon & := \frac1N \sum_{i=1}^N \E\sbr*{ \inf_{g \in \bW} \norm{ g - \varphi_i }_H^2 } .
  \end{align*}
  Then
  \begin{align*}
    r & \geq N \cdot \frac{1-\epsilon}{1 + \sqrt{\sum_{i \neq j} \ip*{\varphi_i,\varphi_j}_H^2}} .
  \end{align*}
  Equality holds when the $\varphi_i$ form an orthonormal basis for $H$.
\end{theorem}
\begin{proof}
  Let $\bu_1,\dotsc,\bu_\bd$ be an orthonormal basis for $\bW$, with $\bd := \dim(\bW)$.
  Let $\Pi_\bW$ denote the orthogonal projection operator for $\bW$.
  Then
  \begin{align*}
    \epsilon
    & = \frac1N \sum_{i=1}^N \E\sbr*{ \inf_{g \in \bW} \norm{ g - \varphi_i }_H^2 }
    && \text{(definition)} \\
    & = \frac1N \sum_{i=1}^N \E\sbr*{ 1 - \norm{ \Pi_\bW\varphi_i }_H^2 }
    && \text{(Hilbert projection theorem)} \\
    & = 1 - \frac1N \E\sbr*{ \sum_{i=1}^N \sum_{k=1}^\bd \ip*{\bu_k,\varphi_i}_H^2 }
    && \text{(linearity of expectation)} \\
    & = 1 - \frac1N \E\sbr*{ \sum_{k=1}^\bd \sum_{i=1}^N \ip*{\bu_k,\varphi_i}_H^2 }
    && \text{(switching order of summations)} \\
    & \geq 1 - \frac1N \E\sbr*{ \sum_{k=1}^\bd \del*{ 1 + \sqrt{ \sum_{i \neq j} \ip*{\varphi_i,\varphi_j}_H^2 } } }
    && \text{(\Cref{fact:boas-bellman})} \\
    & = 1 - \frac{r}{N} \del*{1 + \sqrt{\sum_{i \neq j} \ip*{\varphi_i,\varphi_j}_H^2}}
    && \text{(linearity of expectation)} .
  \end{align*}
  By Parseval's identity, the inequality holds with equality when the $\varphi_i$ form an orthonormal basis for $H$.
\end{proof}

\begin{fact}[\citealp{boas41general,bellman44almost}]
  \label{fact:boas-bellman}
  For any $g, \varphi_1, \dotsc, \varphi_N$ in an inner product space,
  \begin{equation*}
    \sum_{i=1}^N \ip{g, \varphi_i}^2
    \leq \ip{g,g}^2 \del*{ \max_{1 \leq i \leq N} \ip{\varphi_i,\varphi_i}^2 + \sqrt{ \sum_{i \neq j} \ip*{\varphi_i,\varphi_j}^2 }} .
  \end{equation*}
\end{fact}

\section{Lower bounds for kernel methods}

We can use \Cref{thm:dimension} to give a sample size lower bound for kernel methods.
A kernel method based on $n$ training examples $(x_1,y_1),\dotsc,(x_n,y_n) \in \X \times \Y$ returns a function of the form
\begin{align*}
  x & \mapsto \sum_{i=1}^n \alpha_i K(x,x_i)
\end{align*}
for some $\alpha_1,\dotsc,\alpha_n \in \R$ (which may depend on the training examples).
Here $K$ is a positive definite kernel function on the input space $\X$, and $\Y$ is the output space (e.g., $\{-1,1\}$).
The subspace of such functions has dimension at most $n$.
Let $H = L^2(P)$ where $P$ is the probability distribution on $\X$ that we care about, and let $\varphi_1,\dotsc,\varphi_N$ be orthonormal functions in $H$.
If a kernel method can guarantee expected mean squared error at most $\epsilon$ for every $\varphi_i$, then by \Cref{thm:dimension}, the sample size $n$ must be at least $(1-\epsilon)N$.

Note that the argument above holds as long as the subspace does not depend on the target function to be approximated.
A typical approach is to obtain $\bx_1,\dotsc,\bx_n$ as an iid sample from $P$, in which case a kernel method chooses a function from a (random) subspace $\bW$ defined to be the span of the $n$ the functions $x \mapsto K(x,\bx_i)$ for $i=1,\dotsc,n$.
The choice of the function within $\bW$ is typically guided by the labels $y_1,\dotsc,y_n$, and the labels may depend on the target function to be approximated (e.g., $y_i = \varphi_j(\bx_i)$ for all $i=1,\dotsc,n$ if $\varphi_j$ is the target function).
However, the above argument also applies even if the $\bx_1,\dotsc,\bx_n$ are selected deterministically or in any other way, with the corresponding labels $y_1,\dotsc,y_n$ only being revealed after commiting to these $\bx_i$.
This model of learning is a form of learning with membership queries~\citep{angluin1988queries} where the queries are restricted to be non-adaptive.

\paragraph{Example: learning parity functions.}

As a simple example, take $H = L^2(P)$ where $P$ is the uniform distribution on the discrete hypercube $\cbr{-1,1}^d$, and let $\varphi_1,\dotsc,\varphi_N$ be the $N=2^d$ parity functions (which take values in $\cbr{-1,1}$).
The parity functions form an orthonormal basis for $H$.
\Cref{thm:dimension} implies that every kernel method needs $n \geq (1-\epsilon)2^d$ in order to guarantee expected mean squared error $\epsilon$ against every parity function.
Or, let $\varphi_1,\dotsc,\varphi_N$ be the $N=\binom{d}{k}$ parity functions that are $k$-sparse (i.e., only involve $k$ variables).
Then \Cref{thm:dimension} implies a lower bound of $n \geq (1-\epsilon) \binom{d}{k}$ for the same against $k$-sparse parity functions.
As mentioned above, these lower bounds hold even if the kernel method is granted non-adaptive membership queries.

An interesting aspect of this lower bound for kernel methods was pointed out by \citet{bubeck2020provable}, following the work of \citet{allen2020backward}.
Specifically, there are efficient algorithms (which are not kernel methods) for learning any parity function in the non-adaptive membership query model, that run in $\poly(d,1/\epsilon)$ time and use sample size $n = \poly(d,1/\epsilon)$.
Moreover, the learning guarantee holds even if the labels $y_i$ are corrupted by noise in the manner of the classification noise model of \citet{angluin1988learning}.
Such efficient algorithms are not known in the usual statistical model without membership queries (and are conjectured not to exist).
Also, the class of $k$-sparse parity functions can be learned in $\poly(d,1/\epsilon)$ time and with sample size $n = \poly(k,\log d,1/\epsilon)$ \citep{feldman2007attribute}, again, in the non-adaptive membership query model.
It is not known how to achieve this without using membership queries.
So, the lower bounds we described above (based on \Cref{thm:dimension}) imply that kernel methods are not able to benefit from non-adaptive membership queries in the same way that general polynomial-time learning algorithms are.

\bibliographystyle{plainnat}
\bibliography{refs}

\end{document}